\documentclass[11pt]{article}

\usepackage[utf8]{inputenc}
\usepackage[T1]{fontenc}
\usepackage{amsmath,amssymb,amsthm}
\usepackage{booktabs}
\usepackage{graphicx}
\usepackage{xcolor}
\usepackage[hidelinks]{hyperref}
\usepackage{algorithm}
\usepackage{algpseudocode}
\usepackage{tikz}
\usetikzlibrary{arrows.meta,positioning,fit,backgrounds,calc}
\usepackage{longtable}
\usepackage[margin=1in]{geometry}

\theoremstyle{definition}
\newtheorem{definition}{Definition}
\theoremstyle{plain}
\newtheorem{proposition}{Proposition}

\newcommand{\E}{E_\theta}
\newcommand{\lamC}{\lambda_C}
\newcommand{\tsame}{\tau_{\mathrm{same}}}
\newcommand{\twrong}{\tau_{\mathrm{wrong}}}
\newcommand{\tnoise}{\tau_{\mathrm{noise}}}
\newcommand{\Cand}{\mathcal{C}}

\tikzset{
  box/.style={rounded corners=2pt, draw=black!70, line width=0.6pt, align=center,
              inner sep=4pt, font=\small},
  bctx/.style={box, fill=blue!12},
  blam/.style={box, fill=violet!14},
  bsame/.style={box, fill=green!14},
  bwrong/.style={box, fill=orange!18},
  bnoise/.style={box, fill=black!8},
  bscore/.style={box, fill=purple!10},
  flow/.style={-{Triangle[width=6pt,length=6pt]}, line width=1.1pt, draw=black!70},
}

\title{Operator-Mismatch Contrastive Learning for\\ World-Relative Trajectory Coherence}

\author{
One Octadion\thanks{Corresponding author: \texttt{octadion@student.ub.ac.id}} \quad
Firmansyah Sundana\thanks{\texttt{sundana@student.ub.ac.id}} \quad
Novanto Yudistira\thanks{\texttt{yudistira@ub.ac.id}} \\[2pt]
Faculty of Computer Science, Universitas Brawijaya, Indonesia
}

\date{}

\begin{document}
\maketitle

\begin{abstract}
World models are increasingly judged by how well they rank \emph{plausible}
futures, yet plausibility alone is the wrong target. The same scene can
evolve under many latent rules, and a future that is smooth and physically
valid under one rule can be incoherent under the rule actually in force. We
formalize the resulting failure---\emph{operator-blindness}---and show that
it is systematically hidden by standard evaluation: a scorer can attain
near-perfect generic ranking while separating current-world from
other-world futures no better than chance. We introduce
\emph{Operator-Mismatch Contrastive} (OMC) learning, whose semantic hard
negatives are \emph{valid-but-wrong-operator} trajectories, and a matching
protocol that measures world-relative coherence directly. On a controlled
synthetic benchmark with in- and out-of-distribution operator parameters,
across two trajectory-scoring backbones, and on feature-level external
validation with CoPhy \texttt{ballsCF}, OMC raises operator-sensitive
scoring far above a chance baseline that nonetheless achieves
$0.99$ generic top-1. We prove that the wrong-operator negative is necessary
both for the protocol to detect the failure and for the objective to fix it,
and position OMC as an auxiliary coherence \emph{signal}---a trajectory-coherence
scorer, not a generative model, planner, or control agent---that can be attached
to existing world-model representations.
\end{abstract}

\section{Introduction}
\label{sec:intro}

A world model should anticipate futures, and we usually grade it by whether
it ranks a true continuation above corrupted or random alternatives. This
target---\emph{generic plausibility}---is necessary but quietly
insufficient. The same multi-object scene can evolve under many latent
rules (gravity, damping, attraction, restitution), and a future that is
perfectly plausible under one rule can be incoherent under the rule actually
governing the present context. A world model must therefore answer not only
``is this future plausible?'' but ``\emph{under which world is it
coherent?}''

The distinction mirrors an asymmetry in human cognition. People reason about
what is factually plausible, but they also imagine freely---recombining
familiar elements into scenes that need not match the present world, yet
remain internally consistent under different rules
\cite{lake2017,battaglia2013,riochet2021,scholkopf2021}. A capable world model arguably needs both
faculties: the generativity to entertain plausible-elsewhere futures and the
discernment to recognize which of them is coherent here. This paper isolates
the second, measurable faculty---world-relative coherence---and treats free
generativity as motivation rather than claim.

Where should coherence come from? Contemporary world models acquire physical
regularities in two broad ways. One line learns them \emph{implicitly} from
data: joint-embedding predictive architectures acquire intuitive physics
from natural video without hardwired core knowledge
\cite{lecun2022,garrido2025,assran2023}, and latent-dynamics models learn dynamics from
interaction \cite{ha2018,hafner2020}. A complementary line imposes physical
structure \emph{explicitly} via physics-informed losses or conservation-law
architectures \cite{raissi2019,greydanus2019,cranmer2020}. Our contribution
is orthogonal to this axis: whether physics is learned implicitly or imposed
explicitly, the resulting scorer can still be \emph{operator-blind}, judging
a future plausible without registering that it is coherent only under a
different world than the present one.

We make the failure precise (Section~\ref{sec:formulation}). A context $C$
induces a latent operator $\lamC$, and a candidate future $\tau$ is scored
by an energy $\E(\lamC,\tau)$, lower meaning more coherent. We distinguish
\emph{coherent-here} ($\tsame$, valid under the current operator),
\emph{plausible-elsewhere} ($\twrong$, valid under a different operator), and
\emph{incoherent} ($\tnoise$). Operator-blindness is the inability to order
$\tsame$ below $\twrong$ even when noise is rejected---and, as we show
formally, it is invisible to any generic metric whose candidate set omits
$\twrong$.

Our remedy is a change in \emph{what counts as a negative}. OMC inserts
valid futures from a \emph{different} operator as semantic hard negatives in
a contrastive ranking objective (Section~\ref{sec:method}). Such a negative
is a legitimate physical rollout---wrong only \emph{here}---so contrasting
against it forces the scorer to track the inferred operator rather than
surface plausibility.

\paragraph{Contributions.}
\begin{itemize}
  \item \textbf{A failure mode, formalized.} We define operator-blindness and
  prove that generic top-1 ranking cannot detect it whenever the candidate
  set excludes wrong-operator futures (Section~\ref{sec:formulation},
  Prop.~\ref{prop:insens}).
  \item \textbf{An objective and a protocol.} OMC uses wrong-operator hard
  negatives; we pair it with operator-sensitive ranking, coherence
  stratification, and an operator$\times$trajectory energy matrix, and prove
  that the wrong-operator negative is what drives the trained ordering
  (Section~\ref{sec:method}, Prop.~\ref{prop:order}).
  \item \textbf{Evidence.} On synthetic IID/OOD operator parameters, two
  backbones, and feature-level CoPhy \texttt{ballsCF}, OMC lifts
  operator-sensitive metrics well above a baseline that retains
  $\approx\!0.99$ generic top-1
  (Sections~\ref{sec:synthetic}--\ref{sec:cophy}).
\end{itemize}
OMC is a coherence signal for trajectory scoring, not a generative
architecture or planner; we substantiate exactly the relational claim above
and delimit the rest in Section~\ref{sec:limits}. In one line: generic top-1
answers \emph{``is this future plausible?''}, whereas OMC asks \emph{``is it
coherent \underline{here}?''}---and our central finding is that the first
question can be answered almost perfectly while the second is answered at chance.

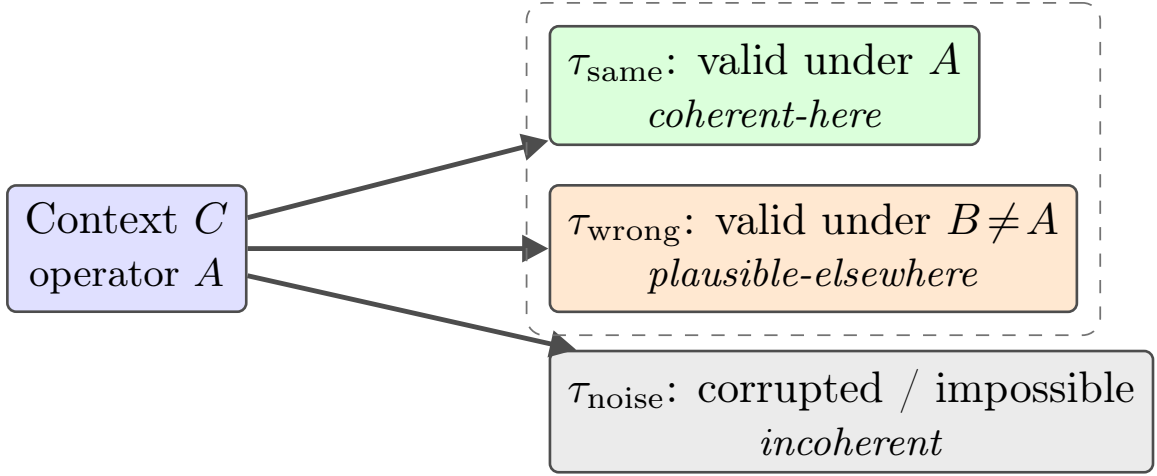
\begin{figure}[t]
\centering
\resizebox{0.92\linewidth}{!}{%
\begin{tikzpicture}[node distance=9mm]
  \node[bctx] (ctx) {Context $C$\\\footnotesize operator $A$};
  \node[bsame, right=24mm of ctx, yshift=13mm] (same) {$\tsame$: valid under $A$\\\footnotesize\textit{coherent-here}};
  \node[bwrong, right=24mm of ctx] (wrong) {$\twrong$: valid under $B\!\neq\!A$\\\footnotesize\textit{plausible-elsewhere}};
  \node[bnoise, right=24mm of ctx, yshift=-13mm] (noise) {$\tnoise$: corrupted / impossible\\\footnotesize\textit{incoherent}};
  \draw[flow] (ctx) -- (same);
  \draw[flow] (ctx) -- (wrong);
  \draw[flow] (ctx) -- (noise);
  \node[draw=black!55, rounded corners, dashed, fit=(same)(wrong), inner sep=5pt] {};
\end{tikzpicture}}
\caption{\textbf{Operator-blindness.} Under operator $A$, both $\tsame$ and
$\twrong$ (valid under a \emph{different} operator $B$) are physically
plausible; only $\tsame$ is coherent here. A scorer that merely rejects
$\tnoise$ may still fail to separate the dashed pair---and, by
Prop.~\ref{prop:insens}, generic ranking cannot reveal it.}
\label{fig:failure}
\end{figure}

\section{Related Work}
\label{sec:related}

\paragraph{World models.}
Latent-dynamics and Dreamer-style models learn dynamics and imagine rollouts
for control \cite{ha2018,hafner2019,hafner2020,hafner2023}; joint-embedding predictive
architectures learn predictive representations and acquire intuitive physics
from video \cite{lecun2022,garrido2025,assran2023,bardes2024}. These are
representational/generative approaches; OMC is a complementary
\emph{objective/evaluation} signal that can sit atop such representations,
and the latent-operator inference we use is standard rather than a
contribution.

\paragraph{Explicit physical structure.}
A parallel tradition encodes physics through physics-informed losses
\cite{raissi2019} or conservation-law architectures
\cite{greydanus2019,cranmer2020}. OMC neither imposes nor presupposes such
structure; it asks the scorer to \emph{learn} a relative coherence signal,
and our results show that neither implicit nor explicit physics by itself
yields operator-sensitive scoring.

\paragraph{Contrastive and energy-based models.}
Structured contrastive models such as C-SWM \cite{kipf2020} and the InfoNCE
family \cite{hadsell2006,oord2018,gutmann2010,chen2020,he2020,robinson2021,chuang2020,locatello2020,lecun2006} learn representations from contrastive
negatives. OMC differs not in using a contrastive energy---that is
standard---but in the \emph{semantics of the hard negative}: a valid
trajectory from a mismatched operator. Hard-negative mining \cite{robinson2021} sharpens such objectives, but its negatives are defined by embedding similarity rather than operator identity. Ordinary structured negatives
(random, state-corrupted, temporal) do not, by themselves, induce
operator-sensitive scoring (Section~\ref{sec:baselines}). CoPhy
\cite{baradel2020} (see also CLEVRER \cite{yi2020}) supplies counterfactual mechanics with hidden
confounders, a natural external testbed used here for validation.

\paragraph{System identification and context inference.}
A large literature infers environment parameters or latent task context for
prediction and control---system identification, adaptive control, and
meta-reinforcement-learning methods that encode task/dynamics context for fast
adaptation \cite{ljung1999,finn2017,rakelly2019,zintgraf2020}. These aim to \emph{use}
the inferred context to act; OMC instead asks whether a scorer's compatibility
function is \emph{sensitive} to that context at all, and supplies a
negative-construction and protocol to enforce and measure it. Interactive world models and benchmarks \cite{bruce2024} emphasize action-conditioned controllability and long-horizon rollout; OMC is complementary, targeting operator-relative coherence of the scorer rather than control.

\begin{table}[t]
\centering\footnotesize\setlength{\tabcolsep}{5pt}
\caption{Positioning of OMC against related method families (citations in the
text of this section). ``Wrong-operator future'' marks whether a method ever
contrasts against a \emph{valid} future from a \emph{different} operator;
``World-relative coherence'' marks whether evaluation tests coherence under the
in-force operator (not merely generic plausibility).
$\checkmark$ yes, $\times$ no, $\sim$ partial/implicit.}
\label{tab:related}
\resizebox{\ifdim\width>\linewidth\linewidth\else\width\fi}{!}{%
\begin{tabular}{lcccc}
\toprule
Method family & Learns & Contrastive & Wrong-operator & World-relative\\
 & dynamics & negatives & future & coherence\\
\midrule
Latent-dynamics / Dreamer & $\checkmark$ & $\times$ & $\times$ & $\times$\\
JEPA / I-JEPA & $\sim$ & $\sim$ & $\times$ & $\times$\\
C-SWM / structured contrastive & $\checkmark$ & $\checkmark$ & $\times$ & $\times$\\
CPC / InfoNCE family & $\sim$ & $\checkmark$ & $\times$ & $\times$\\
Physics-informed / HNN / LNN & $\checkmark$ & $\times$ & $\times$ & $\times$\\
Meta-RL context inference & $\sim$ & $\times$ & $\times$ & $\times$\\
\midrule
\textbf{OMC (ours)} & $\times$ & $\checkmark$ & $\checkmark$ & $\checkmark$\\
\bottomrule
\end{tabular}}
\end{table}

\section{Problem Formulation}
\label{sec:formulation}

\paragraph{Generative process.}
A latent operator $\lambda^\star\!\sim\!p(\lambda)$ parameterizes the
dynamics; a context $C$ and its true future are draws from $p(\cdot\mid
\lambda^\star)$. A context encoder infers a posterior over the operator,
\begin{equation}
  q_\phi(\lambda \mid C) = \mathcal{N}\!\big(\mu_\phi(C),\,
  \mathrm{diag}\,\sigma^2_\phi(C)\big),
  \qquad \lamC \triangleq \mu_\phi(C),
\end{equation}
and a scorer assigns an energy $\E(\lamC,\tau)\in\mathbb{R}$ to a candidate
future $\tau=(s_0,\dots,s_T)$, with \emph{lower energy denoting greater
coherence}. Coherence is identified with operator-conditional likelihood:
$\tau$ is coherent under $\lambda$ to the extent that $p(\tau\mid\lambda)$ is
high, the ideal scorer tracking $\E(\lambda,\tau)\propto-\log
p(\tau\mid\lambda)$.

\paragraph{What the operator can be.}
We keep $\lambda$ abstract on purpose. Although our experiments instantiate it
as simulator parameters, the same formulation covers any latent that selects
which dynamics are in force: a latent dynamics code, a hidden confounder (as in
CoPhy), an environmental law or physical regime, a task/parameter context in
meta-RL, or an intervention setting in a causal model. Operator-blindness is
therefore not a quirk of one simulator but a generic failure of any scorer whose
compatibility function is insensitive to the latent that governs the present
context.

\paragraph{Candidate classes.}
From the generative process we define
\begin{align}
  \tsame &\sim p(\cdot\mid\lambda^\star) &&\text{(coherent-here; the positive),}\nonumber\\
  \twrong &\sim p(\cdot\mid\lambda'),\ \lambda'\neq\lambda^\star &&\text{(plausible-elsewhere; operator/law mismatch),}\nonumber\\
  \tnoise &\sim p_{\mathrm{corrupt}} &&\text{(incoherent: corrupted, time-shuffled, or impossible).}\nonumber
\end{align}
A world-relatively coherent scorer should satisfy
\begin{equation}
  \label{eq:ordering}
  \E(\lamC, \tsame) \;<\; \E(\lamC, \twrong) \;<\; \E(\lamC, \tnoise).
\end{equation}

\begin{definition}[Operator-blindness]
\label{def:blind}
A scorer is \emph{operator-blind} at $C$ if
\begin{align*}
  \Pr[\E(\lamC,\tsame)<\E(\lamC,\tnoise)] &\to 1,\\
  \Pr[\E(\lamC,\tsame)<\E(\lamC,\twrong)] &= \tfrac12,
\end{align*}
i.e.\ it rejects noise but orders same- and wrong-operator futures at chance.
\end{definition}

The next statement is why operator-blindness goes unnoticed: it is invisible
to any generic ranking metric whose candidate pool excludes $\twrong$.

\begin{proposition}[Generic ranking cannot detect operator-blindness]
\label{prop:insens}
Let the generic candidate set be $\{\tsame\}\cup\mathcal{N}_{\mathrm{easy}}$
with $\mathcal{N}_{\mathrm{easy}}$ containing only corrupted/temporal/random
negatives, and suppose $\E(\lamC,\tsame)<\E(\lamC,e)$ a.s.\ for every
$e\in\mathcal{N}_{\mathrm{easy}}$. Then $\mathrm{top1\text{-}acc}=1$
regardless of the ordering of $\E(\lamC,\tsame)$ and $\E(\lamC,\twrong)$. In
particular an operator-blind scorer and a perfectly coherent scorer are
indistinguishable under generic top-1.
\end{proposition}
\begin{proof}
Top-1 accuracy is determined solely by whether $\tsame$ attains the minimum
energy within $\{\tsame\}\cup\mathcal{N}_{\mathrm{easy}}$; since $\twrong$ is
absent from this set, the same/wrong ordering does not enter the event. The
hypothesis makes the event certain, so $\mathrm{top1\text{-}acc}=1$ for both
scorers.
\end{proof}

Proposition~\ref{prop:insens} motivates the protocol of
Section~\ref{sec:setup}: to detect operator-blindness one must explicitly
insert $\twrong$ into the comparison, which is exactly what \textsc{law\_pair}
and \textsc{same\_lt\_wrong} do.

\section{Operator-Mismatch Contrastive Learning}
\label{sec:method}

\paragraph{Scorer and operator encoder.}
The context encoder is a permutation-invariant (DeepSets \cite{zaheer2017}) network over the
context transitions, producing $\mu_\phi(C),\log\sigma^2_\phi(C)$ with latent
dimension $d_\lambda$. The energy scorer is object-centric, in the spirit of interaction networks \cite{battaglia2016} and slot-based models \cite{locatello2020}: with $s^i_t$ the
state of object $i$ at step $t$ and $a^i_t$ its action, a per-object
\emph{self} term scores each transition under the operator code, and an
optional \emph{pairwise} term scores object pairs,
\begin{equation}
  \label{eq:energy}
  \E(\lamC, \tau) =
  \sum_{t}\sum_{i} f_{\text{self}}\!\big(s^i_t, a^i_t, s^i_{t+1}, \lamC\big)
  + \underbrace{\sum_{t}\sum_{i<j} f_{\text{rel}}\!\big(s^i_t, s^j_t, s^i_{t+1}, s^j_{t+1}, \lamC\big)}_{\text{pairwise head (optional)}},
\end{equation}
with masking over inactive slots and a sum reduction. The pairwise term is an
architectural option, disabled in the main objective and studied as a
component ablation (Section~\ref{sec:synthetic}). A \emph{direct-context}
variant keeps the same objective but drops the explicit operator bottleneck,
isolating the objective from the backbone.

\paragraph{Candidate set.}
For each context we form $\Cand$ of size $M=5$: the positive plus one
negative of each type,
$\Cand=\{\tsame,\tau^{\text{state}},\tau^{\text{temporal}},\twrong,\tau^{\text{rand}}\}$,
where $\tau^{\text{state}}$ perturbs future positions/velocities,
$\tau^{\text{temporal}}$ reverses or permutes the future, $\tau^{\text{rand}}$
is impossible, and $\twrong$---the OMC negative---is a valid future from a
\emph{different} operator (distinct identity or parameter distance
$>\delta$). Batches contain distinct operators, so each $\twrong$ is a
genuine wrong-operator rollout.

\paragraph{Objective.}
We minimize an InfoNCE ranking loss with logits equal to negative energies,
\begin{equation}
  \label{eq:nce}
  \mathcal{L}_{\text{NCE}}
  = -\log
  \frac{\exp(-\E(\lamC,\tsame))}
       {\sum_{\tau \in \Cand}\exp(-\E(\lamC,\tau))},
\end{equation}
regularized by a KL term on the operator posterior \cite{kingma2014} and a stability term
tying two halves of a context to the same operator,
$\mathcal{L}_{\text{KL}}=\mathrm{KL}(q_\phi(\lambda\mid
C)\,\|\,\mathcal{N}(0,I))$ and
$\mathcal{L}_{\text{stable}}=\|\mu_\phi(C_{1:k})-\mu_\phi(C_{k+1:})\|_2^2$.
The total objective is
\begin{equation}
  \label{eq:total}
  \mathcal{L}=\mathcal{L}_{\text{NCE}}
  +\alpha_{\text{occl}}\mathcal{L}_{\text{OCCL}}
  +\beta_{\text{KL}}\mathcal{L}_{\text{KL}}
  +\alpha_{\text{stable}}\mathcal{L}_{\text{stable}},
\end{equation}
where $\mathcal{L}_{\text{OCCL}}$ is a symmetric InfoNCE over an
operator$\times$trajectory energy matrix (Appendix~\ref{app:objective}). In
all reported runs $\alpha_{\text{occl}}=0$: the matrix term is computed for
diagnosis but does not enter the gradient, so the effect we report comes
\emph{entirely} from the wrong-operator negative in
Eq.~\eqref{eq:nce}; the energy matrix is an evaluation instrument, not a
trained objective. The only further axis we vary is the architectural
pairwise head of Eq.~\eqref{eq:energy}.

The wrong-operator negative is not merely helpful but \emph{necessary} for
the objective to induce the target ordering:

\begin{proposition}[The wrong-operator negative induces coherence]
\label{prop:order}
For a sufficiently expressive scorer and in the population limit, the
minimizer of Eq.~\eqref{eq:nce} satisfies $-\E(\lamC,\tau)=\log
p(\tau\mid\lamC)+c$ for a candidate-independent $c$; hence
$\E(\lamC,\tsame)<\E(\lamC,\twrong)$ whenever
$p(\tsame\mid\lamC)>p(\twrong\mid\lamC)$. Conversely, if $\twrong\notin\Cand$
the objective places no constraint on the relative energies of same- and
wrong-operator futures, admitting operator-blind minimizers.
\end{proposition}
\begin{proof}[Proof sketch]
Softmax cross-entropy over candidates is minimized when the induced
selection distribution matches the true posterior that the positive is the
in-context continuation; the noise-contrastive optimality result
\cite{gutmann2010,oord2018} gives optimal scores $s(\tau)=\log
p(\tau\mid\lamC)$ up to an additive constant, and $s=-\E$ yields the ordering.
If $\twrong$ never appears in $\Cand$, the loss is invariant to the value of
$\E(\lamC,\twrong)$, so minimizers need not satisfy
Eq.~\eqref{eq:ordering}---the operator-blind regime. Together with
Prop.~\ref{prop:insens}, the wrong-operator trajectory is necessary both to
\emph{train} and to \emph{detect} world-relative coherence.
\end{proof}

\paragraph{Training-time requirement and inference-time usage.}
A natural objection is that constructing $\twrong$ presupposes access to
alternative-operator trajectories. We stress that this is a \emph{training- and
evaluation-time} requirement only. Wrong-operator rollouts are needed to build
the contrastive candidate set during learning and to probe operator sensitivity
during evaluation; they are \emph{not} required at inference. Once trained, the
scorer consumes a context $C$ and a single candidate $\tau$ and returns an energy
$\E(\lamC,\tau)$ with no alternative-operator data in the loop. OMC is thus an
attachable coherence signal: the cost it imposes is the availability of
operator-labelled rollouts at train time, exactly the supervision our synthetic
generator and CoPhy's hidden confounder already provide.

\begin{algorithm}[t]
\caption{OMC training step (one minibatch)}
\label{alg:omc}
\begin{algorithmic}[1]
\State sample a minibatch with \textbf{distinct operators} across samples
\For{each sample $i$}
  \State $C_i\gets$ context; $\tsame^{(i)}\gets$ valid future under the current operator
  \State negatives $\gets \tau^{\text{state}},\tau^{\text{temporal}},\twrong,\tau^{\text{rand}}$
         \quad\textit{// $\twrong$: valid future from a different operator}
  \State $\Cand_i\gets[\tsame^{(i)},\tau^{\text{state}},\tau^{\text{temporal}},\twrong,\tau^{\text{rand}}]$
\EndFor
\State $\mu,\log\sigma^2\gets\text{Enc}(C)$;\ $\lambda\gets\mu+\sigma\odot\epsilon$;\
       $\E[i,m]\gets\text{Score}(\lambda_i,\Cand_i[m])$
\State $\mathcal{L}\gets\mathrm{CE}(-\E,y)+\beta_{\text{KL}}\mathcal{L}_{\text{KL}}+\alpha_{\text{stable}}\mathcal{L}_{\text{stable}}$
       \quad($\alpha_{\text{occl}}{=}0$)
\State AdamW update on $\theta,\phi$
\end{algorithmic}
\end{algorithm}

\begin{figure}[t]
\centering
\resizebox{\linewidth}{!}{%
\begin{tikzpicture}
  \node[bctx] (ctx) {Context $C$};
  \node[box, right=9mm of ctx, fill=blue!8] (enc) {Context encoder\\\footnotesize $q_\phi(\lambda\mid C)$};
  \node[blam, right=9mm of enc] (lam) {operator\\code $\lamC$};
  \node[bscore, right=11mm of lam] (score) {Energy scorer\\\footnotesize $\E(\lamC,\tau)$};
  \node[box, right=11mm of score, fill=purple!8] (nce) {InfoNCE\\\footnotesize over $-\E$};
  \node[bsame, above=7mm of score, text width=60mm] (cand)
        {candidates $\Cand=\{\tsame,\tau^{\text{state}},\tau^{\text{temporal}},\mathbf{\twrong},\tau^{\text{rand}}\}$};
  \draw[flow] (ctx)--(enc); \draw[flow] (enc)--(lam); \draw[flow] (lam)--(score);
  \draw[flow] (cand)--(score); \draw[flow] (score)--(nce);
\end{tikzpicture}}
\caption{\textbf{OMC pipeline.} A context encoder infers $\lamC$; an
object-centric energy scorer ranks the candidate set with InfoNCE. The
defining choice is the wrong-operator hard negative $\twrong$. The
operator$\times$trajectory energy matrix / OCCL term is a diagnostic only
($\alpha_{\text{occl}}{=}0$).}
\label{fig:pipeline}
\end{figure}
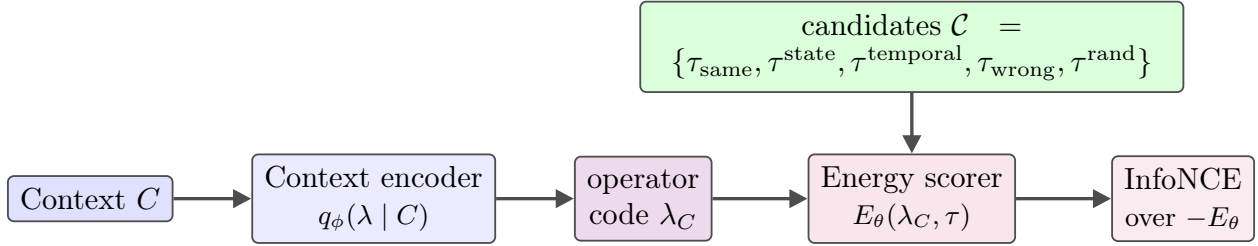

\section{Experimental Setup}
\label{sec:setup}

\paragraph{Synthetic benchmark.}
\textsc{LOWM-Synth} is a 2D multi-object simulator (horizon $T{=}20$) whose
hidden operator is drawn from four families (gravity/damping,
attraction/repulsion, collision elasticity, boundary behavior) with
continuous parameters. \texttt{test\_iid} samples operators from the
training ranges; \texttt{test\_ood\_param} uses held-out parameter ranges of
the same families, testing whether coherence generalizes beyond memorized
parameters. Operators are metadata used to build negatives and to evaluate;
they are never given to the model.

\paragraph{External validation and backbones.}
We additionally validate on CoPhy \texttt{ballsCF} \cite{baradel2020} at the
feature level (segmentation-derived object features, $64$px, $20$ frames;
$35$k/$7.5$k/$7.5$k split). Two backbones, \textsc{LOWM} (latent operator)
and \textsc{direct-context} (no operator bottleneck), separate the objective
from the architecture.

\paragraph{Metrics.}
We separate \emph{primary} operator-sensitive metrics from a
\emph{diagnostic} one. Primary: \textsc{law\_pair}
($\Pr[\E(\tsame){<}\E(\twrong)]$), \textsc{same\_lt\_wrong} and
\textsc{full\_stratification} (the orderings of Eq.~\eqref{eq:ordering}), and
energy-matrix \textsc{mrr}/\textsc{top1} (operator$\times$trajectory
retrieval). Diagnostic: \textsc{top1\_acc} over the generic candidate set.
By Prop.~\ref{prop:insens} the diagnostic alone cannot reveal
operator-blindness, so we report it only to expose the trade-off. Formal
definitions are in Appendix~\ref{app:metrics}. \emph{All energy-matrix
metrics use \texttt{matrix\_size}$=16$, \texttt{max\_batches}$=8$} and are
never compared across sizes. We report mean$\pm$std over $5$ seeds
(synthetic/CoPhy) or $3$ (baselines); per-seed values are in
Appendix~\ref{app:perseed}.

\begin{table}[t]
\centering\footnotesize\setlength{\tabcolsep}{5pt}
\caption{Metric$\rightarrow$hypothesis map. Each operator-sensitive metric uses a
different candidate set; \textsc{law\_pair} is defined only when a wrong-operator
negative is present.}
\label{tab:metricmap}
\begin{tabular}{lll}
\toprule
Metric & Candidate set & Reads as\\
\midrule
  \textsc{top1\_acc} & $\{\tsame\}\cup$ easy negs & generic plausibility (diagnostic) \\
  \textsc{law\_pair} & full set \emph{incl.}\ $\twrong$ & $P(\tsame{<}\twrong)$; n/a if $\twrong$ absent \\
  \textsc{same\_lt\_wrong} & $\{\tsame,\twrong,\tnoise\}$ & $P(\tsame{<}\twrong)$, fixed 3-way set \\
  \textsc{full\_stratification} & $\{\tsame,\twrong,\tnoise\}$ & $P(\tsame{<}\twrong{<}\tnoise)$ \\
  \textsc{matrix\_mrr} & $B{\times}B$ operator$\times$traj.\ & operator$\leftrightarrow$trajectory retrieval \\
\bottomrule
\end{tabular}
\end{table}

\noindent\emph{Reading \textsc{law\_pair}.} It is computed only over samples whose
candidate set contains a wrong-operator negative. The \textsc{no-law} baseline has
none, so its \textsc{law\_pair} is reported as $0$/n.a.\ (vacuous)---an instance of
Prop.~\ref{prop:insens}, not a behavior; its genuine chance-level operator
sensitivity is read from \textsc{same\_lt\_wrong}$\approx0.5$ (Table~\ref{tab:compose}).

\section{Synthetic Diagnostics and Ablations}
\label{sec:synthetic}

\paragraph{Reading guide.}
We organize the empirical sections as answers to five questions a reviewer would
ask. \textbf{R1: Does operator-blindness exist?}---Table~\ref{tab:main} and
Figure~\ref{fig:strat}. \textbf{R2: Does OMC fix it?}---Table~\ref{tab:main}, on
both backbones and both IID/OOD splits. \textbf{R3: Is the wrong-operator
negative the active ingredient, rather than contrastive learning in general?}
---the negative-type and composition ablations (Tables~\ref{tab:negtype},
\ref{tab:compose}) and the general world-model baselines (Table~\ref{tab:baselines}).
\textbf{R4: Does it generalize beyond a memorized simulator?}---the OOD-parameter
rows throughout and the external CoPhy validation (Table~\ref{tab:cophy}).
\textbf{R5: How should OMC be used?}---as an attachable verifier signal, discussed
in Section~\ref{sec:limits}. Each table below is tagged with the question it answers.

\paragraph{(R1)~Operator-blindness is real and hidden behind top-1.}
Table~\ref{tab:main} states the central result. The
\textsc{no-law-mismatch} baseline reaches $\approx\!0.99$ generic top-1 yet
\textsc{law\_pair}$=0$ and \textsc{same\_lt\_wrong}$\approx\!0.5$: it rejects
noise almost perfectly while ordering current- vs.\ other-world futures at
chance---an empirical instance of Definition~\ref{def:blind} that
Prop.~\ref{prop:insens} predicts the diagnostic will miss. OMC inverts the
profile, trading generic top-1 (expected for a harder, operator-targeted
objective) for large gains on every operator-sensitive metric. The pattern
holds for \emph{both} backbones, so it is driven by the objective rather than
the architecture, and it persists under the OOD operator split, ruling out
parameter memorization.

\begin{table}[t]
\centering\footnotesize\setlength{\tabcolsep}{4.5pt}
\caption{Main synthetic results ($5$ seeds, mean$\pm$std;
\texttt{matrix\_size}$=16$). \textsc{law\_pair}, \textsc{same\_lt\_wrong},
\textsc{matrix\_mrr} are primary; \textsc{top1} is diagnostic.}
\label{tab:main}
\begin{tabular}{lll rrrrr}
\toprule
Backbone & Obj. & Split & top1 & law\_pair & law\_only & same$<$wr & mat\_mrr\\
\midrule
  LOWM & OMC & IID & 0.375$\pm$0.139 & \textbf{0.694}$\pm$0.005 & 0.411$\pm$0.003 & \textbf{0.697}$\pm$0.006 & \textbf{0.272}$\pm$0.012 \\
  LOWM & OMC & OOD & 0.363$\pm$0.116 & \textbf{0.650}$\pm$0.006 & 0.359$\pm$0.011 & \textbf{0.665}$\pm$0.009 & \textbf{0.271}$\pm$0.009 \\
  LOWM & no-law & IID & 0.996$\pm$0.001 & 0.000 & 0.196$\pm$0.013 & 0.507$\pm$0.010 & 0.213$\pm$0.003 \\
  LOWM & no-law & OOD & 0.991$\pm$0.001 & 0.000 & 0.205$\pm$0.004 & 0.514$\pm$0.007 & 0.209$\pm$0.004 \\
\midrule
  direct & OMC & IID & 0.574$\pm$0.014 & \textbf{0.709}$\pm$0.008 & \textbf{0.419}$\pm$0.008 & \textbf{0.723}$\pm$0.002 & 0.262$\pm$0.001 \\
  direct & OMC & OOD & 0.562$\pm$0.023 & \textbf{0.700}$\pm$0.012 & 0.409$\pm$0.018 & \textbf{0.705}$\pm$0.011 & 0.259$\pm$0.005 \\
  direct & no-law & IID & 0.997$\pm$0.001 & 0.000 & 0.207$\pm$0.007 & 0.508$\pm$0.002 & 0.211$\pm$0.000 \\
  direct & no-law & OOD & 0.996$\pm$0.000 & 0.000 & 0.200$\pm$0.007 & 0.502$\pm$0.009 & 0.211$\pm$0.001 \\
\bottomrule
\end{tabular}
\end{table}

\paragraph{(R3)~The wrong-operator negative is the active ingredient.}
Table~\ref{tab:negtype} ablates the negative set, corroborating
Prop.~\ref{prop:order}. Removing the wrong-operator negative (\textsc{no-law})
collapses operator-sensitive scores to chance while preserving generic
ranking; using it \emph{exclusively} (\textsc{law-only}) maximizes the
targeted metrics. The balanced all-negative objective trades a little of that
peak for broader coverage. This isolates the wrong-operator negative---not
contrastive learning per se---as the source of the effect.

\begin{table}[t]
\centering\scriptsize\setlength{\tabcolsep}{3pt}
\caption{Negative-type ablation (LOWM, $5$ seeds, \texttt{matrix\_size}$=16$).}
\label{tab:negtype}
\begin{tabular}{ll rrrrrr}
\toprule
Objective & Split & top1 & law\_pair & law\_gap & same$<$wr & full str. & mat\_mrr\\
\midrule
  all-neg (main OMC) & IID & 0.375$\pm$0.139 & 0.694$\pm$0.005 & 1.30$\pm$0.60 & 0.697$\pm$0.006 & \textbf{0.697}$\pm$0.006 & 0.272$\pm$0.012 \\
  all-neg (main OMC) & OOD & 0.363$\pm$0.116 & 0.650$\pm$0.006 & 1.21$\pm$0.61 & 0.665$\pm$0.009 & 0.665$\pm$0.009 & 0.271$\pm$0.009 \\
  no-law-mismatch & IID & 0.996$\pm$0.001 & 0.000 & 0.00 & 0.507$\pm$0.010 & 0.507$\pm$0.010 & 0.213$\pm$0.003 \\
  no-law-mismatch & OOD & 0.991$\pm$0.001 & 0.000 & 0.00 & 0.514$\pm$0.007 & 0.514$\pm$0.007 & 0.209$\pm$0.004 \\
  law-mismatch only & IID & 0.416$\pm$0.007 & \textbf{0.700}$\pm$0.004 & \textbf{2.17}$\pm$0.08 & \textbf{0.712}$\pm$0.004 & 0.676$\pm$0.011 & 0.284$\pm$0.012 \\
  law-mismatch only & OOD & 0.403$\pm$0.003 & 0.695$\pm$0.003 & 2.11$\pm$0.11 & 0.692$\pm$0.006 & 0.660$\pm$0.013 & \textbf{0.286}$\pm$0.013 \\
\bottomrule
\end{tabular}
\end{table}

\paragraph{The pairwise head trades off, it does not dominate.}
Table~\ref{tab:component} compares the pairwise-on (\texttt{main\_all}) and
pairwise-off settings under the same all-negative objective. The pairwise
head raises generic top-1 and lowers variance and slightly improves
\textsc{same\_lt\_wrong}, while pairwise-off attains slightly higher
\textsc{matrix\_mrr}. We therefore do not attribute the OMC effect to the
pairwise component; it shifts the trade-off between generic ranking and
matrix retrieval.

\begin{table}[t]
\centering\scriptsize\setlength{\tabcolsep}{4pt}
\caption{Component ablation (LOWM, $5$ seeds, all-negative objective,
\texttt{matrix\_size}$=16$).}
\label{tab:component}
\begin{tabular}{ll rrrrr}
\toprule
Objective & Split & top1 & law\_pair & same$<$wr & mat\_mrr & mat\_top1\\
\midrule
  main\_all (pairwise on) & IID & 0.509$\pm$0.010 & 0.692$\pm$0.003 & \textbf{0.703}$\pm$0.004 & 0.255$\pm$0.007 & 0.092$\pm$0.008 \\
  main\_all (pairwise on) & OOD & 0.482$\pm$0.011 & 0.656$\pm$0.004 & 0.670$\pm$0.003 & 0.257$\pm$0.003 & 0.096$\pm$0.005 \\
  no\_pairwise (main OMC) & IID & 0.375$\pm$0.139 & \textbf{0.694}$\pm$0.005 & 0.697$\pm$0.006 & \textbf{0.272}$\pm$0.012 & \textbf{0.105}$\pm$0.017 \\
  no\_pairwise (main OMC) & OOD & 0.363$\pm$0.116 & 0.650$\pm$0.006 & 0.665$\pm$0.009 & 0.271$\pm$0.009 & 0.103$\pm$0.011 \\
\bottomrule
\end{tabular}
\end{table}

\paragraph{Visualizing the mechanism.}
Figure~\ref{fig:matrix} shows a representative operator$\times$trajectory
energy matrix: under OMC the matched (diagonal) operator gives the lowest
energy within each column, whereas the baseline shows no such structure.
Figure~\ref{fig:strat} shows the corresponding energy distributions: OMC
realizes $\tsame{<}\twrong{<}\tnoise$, while the baseline collapses $\tsame$
and $\twrong$ yet still separates $\tnoise$---operator-blindness made
visible.

\begin{figure}[t]
\centering
\includegraphics[width=\linewidth]{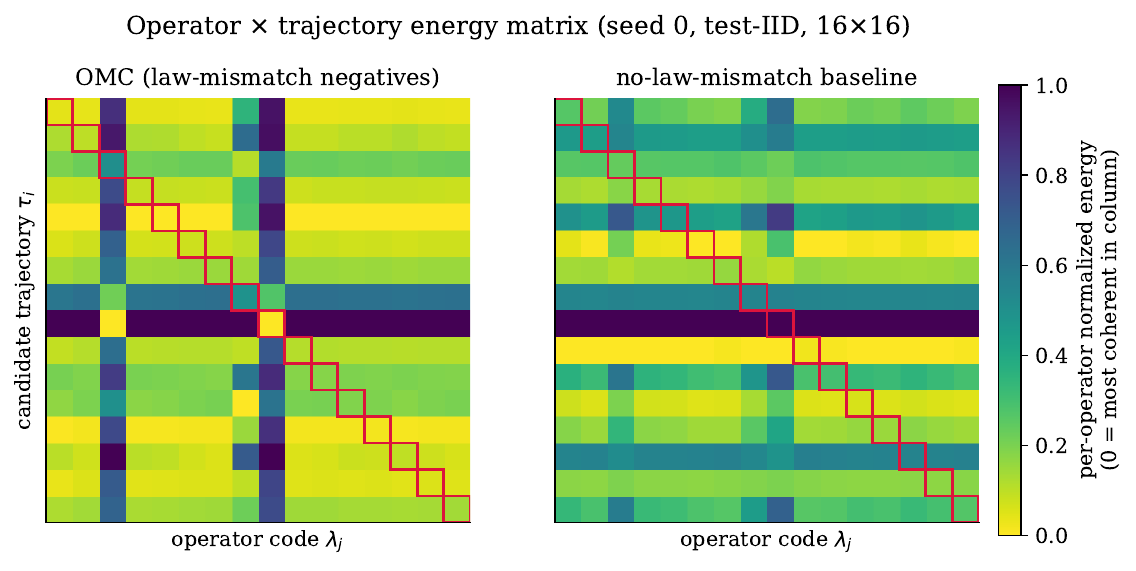}
\caption{\textbf{Energy matrix} (seed 0, IID, $16\times16$), per-operator
normalized ($0=$ most coherent in each column). OMC concentrates low energy
on the diagonal; the baseline does not.}
\label{fig:matrix}
\end{figure}

\begin{figure}[t]
\centering
\includegraphics[width=\linewidth]{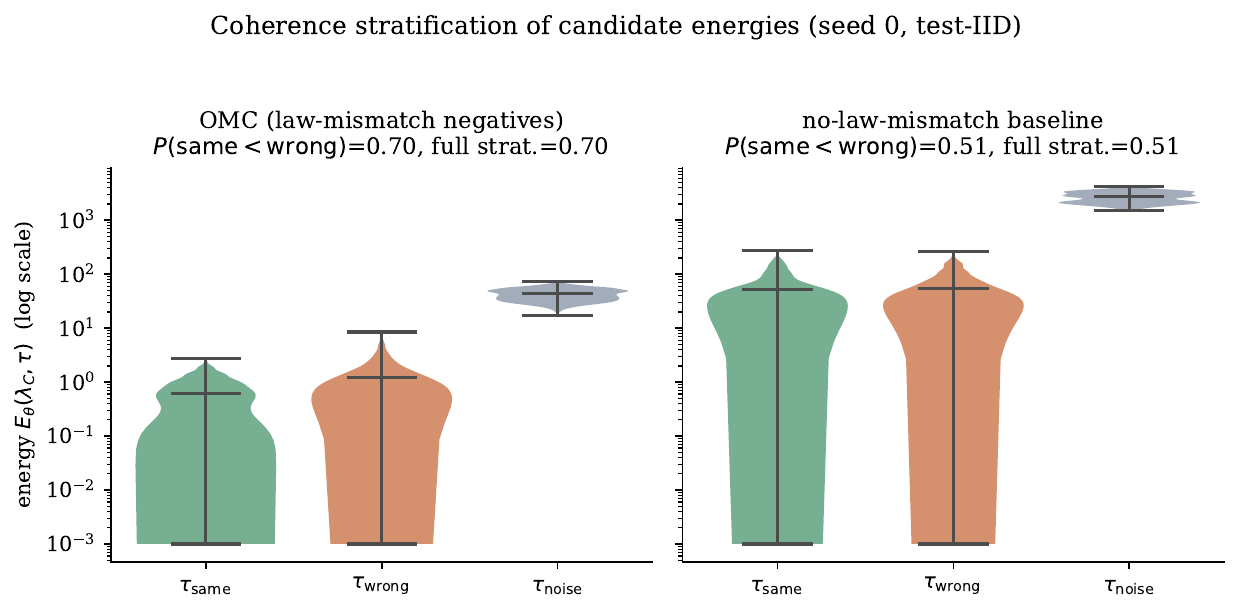}
\caption{\textbf{Coherence stratification} (seed 0, IID). OMC separates
$\tsame$, $\twrong$, $\tnoise$; the baseline collapses same vs.\ wrong while
placing noise far above.}
\label{fig:strat}
\end{figure}

\paragraph{The effect is relational, not operator decodability.}
OMC does not produce a more decodable operator code: linear/MLP probes on
$\lamC$ recover operator identity and binned parameters about equally for OMC
and the baseline (seed-0 IID: linear operator-id $0.60$ vs.\ $0.61$;
binned-parameter $0.89$ vs.\ $0.90$; parameter $R^2$ $0.35$ vs.\ $0.39$).%
\ The gain is therefore relational---OMC shapes the compatibility
$\E(\lamC,\tau)$, consistent with Prop.~\ref{prop:order}, rather than the
operator embedding. Full probe numbers are in Appendix~\ref{app:perseed}.

\subsection{Sensitivity analysis}
\label{sec:sensitivity}
We vary the main design choices around the paper's operating point (no-pairwise,
$\alpha_{\text{occl}}{=}0$, all-4 negatives), changing one knob at a time ($3$ seeds;
the $\alpha_{\text{occl}}{=}0$ / all-4 row reproduces Table~\ref{tab:main}). We report
\texttt{test\_iid}; the OOD split shows the same pattern.

\paragraph{Training the matrix term ($\alpha_{\text{occl}}>0$).}
Weighting the operator-coherence matrix loss (Table~\ref{tab:alpha}) gives a small
additional gain in operator-sensitive metrics (\textsc{law\_pair} $0.694\!\to\!0.710$,
\textsc{matrix\_mrr} $0.269\!\to\!0.289$) at a further cost to generic top-1
($0.38\!\to\!0.24$). Since the wrong-operator negative alone
($\alpha_{\text{occl}}{=}0$) already captures most of the effect, we keep
$\alpha_{\text{occl}}{=}0$ as the simplest sufficient setting; training the matrix
term is a viable, marginally stronger alternative when generic top-1 is not a concern.

\begin{table}[t]
\centering\footnotesize\setlength{\tabcolsep}{6pt}
\caption{Sensitivity to $\alpha_{\text{occl}}$ (LOWM, no-pairwise, all-4 negatives;
$3$ seeds, \texttt{test\_iid}). $\alpha_{\text{occl}}{=}0$ is the paper setting.}
\label{tab:alpha}
\begin{tabular}{r rrrr}
\toprule
$\alpha_{\text{occl}}$ & top1 & law\_pair & same$<$wr & mat\_mrr\\
\midrule
  0.0 & 0.384$\pm$0.125 & 0.694$\pm$0.007 & 0.693$\pm$0.003 & 0.269$\pm$0.004 \\
  0.1 & 0.219$\pm$0.016 & 0.703$\pm$0.006 & 0.700$\pm$0.007 & 0.290$\pm$0.006 \\
  0.5 & 0.230$\pm$0.033 & 0.706$\pm$0.008 & 0.703$\pm$0.004 & \textbf{0.292}$\pm$0.010 \\
  1.0 & 0.239$\pm$0.043 & \textbf{0.710}$\pm$0.009 & \textbf{0.710}$\pm$0.006 & 0.289$\pm$0.005 \\
\bottomrule
\end{tabular}
\end{table}

\paragraph{Negative-set composition.}
Table~\ref{tab:compose} varies the candidate negatives. Any set containing the
wrong-operator negative reaches \textsc{law\_pair}/\textsc{same\_lt\_wrong}
$\approx 0.69$--$0.71$; removing it (\textsc{no-law}) collapses operator sensitivity
to chance (\textsc{same\_lt\_wrong} $0.506$, \textsc{matrix\_mrr} $0.211$) while
generic top-1 stays near-perfect ($0.996$). The non-operator negatives mainly trade
operator sensitivity for generic top-1 (adding \texttt{random} raises top-1 to $0.56$
at no cost to operator metrics), making the OMC top-1 reduction a tunable design
choice rather than an inherent penalty. This is the cleanest confirmation of
Prop.~\ref{prop:order}: the wrong-operator negative is necessary and largely sufficient.

\begin{table}[t]
\centering\footnotesize\setlength{\tabcolsep}{5pt}
\caption{Sensitivity to negative-set composition (LOWM, $\alpha_{\text{occl}}{=}0$,
no-pairwise; $3$ seeds, \texttt{test\_iid}).}
\label{tab:compose}
\begin{tabular}{l rrrr}
\toprule
Negative set & top1 & law\_pair & same$<$wr & mat\_mrr\\
\midrule
  law-mismatch only & 0.419$\pm$0.003 & 0.707$\pm$0.002 & 0.711$\pm$0.007 & 0.289$\pm$0.002 \\
  law-mismatch $+$ random & 0.561$\pm$0.001 & \textbf{0.708}$\pm$0.003 & \textbf{0.713}$\pm$0.003 & \textbf{0.292}$\pm$0.007 \\
  all-4 (main OMC) & 0.384$\pm$0.125 & 0.694$\pm$0.007 & 0.693$\pm$0.003 & 0.269$\pm$0.004 \\
  no-law ($\twrong$ removed) & 0.996$\pm$0.002 & 0.000 & 0.506$\pm$0.014 & 0.211$\pm$0.003 \\
\bottomrule
\end{tabular}
\end{table}

\paragraph{Robustness to $\delta$ and temperature.}
Operator-sensitive metrics are essentially flat across the operator-distance threshold
$\delta\in[0.05,0.4]$ and the InfoNCE temperature $T\in[0.5,4]$
(Figure~\ref{fig:sensitivity}): \textsc{law\_pair} stays at $\approx 0.69$ throughout,
so OMC is insensitive to these hyperparameters within the tested ranges.

\begin{figure}[t]
\centering
\includegraphics[width=\linewidth]{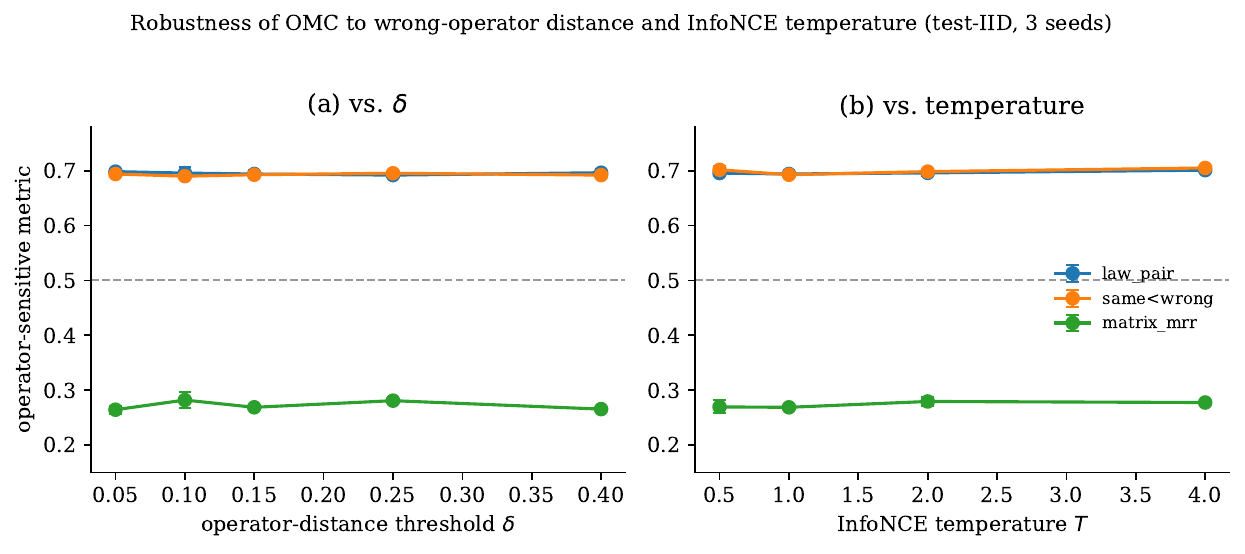}
\caption{\textbf{Robustness} of OMC to wrong-operator distance $\delta$ (left) and
InfoNCE temperature $T$ (right); operator-sensitive metrics stay flat
(\texttt{test\_iid}, $3$ seeds).}
\label{fig:sensitivity}
\end{figure}

\section{General World-Model Baselines}
\label{sec:baselines}

Table~\ref{tab:baselines} places OMC against feature-level C-SWM-style
\cite{kipf2020} and RSSM-lite scorers. With ordinary structured negatives
(random, state-corrupted, temporal), the C-SWM-style scorer stays at chance
on operator-sensitive metrics---contrastive structure alone does not induce
operator sensitivity, exactly as Prop.~\ref{prop:order} predicts when
$\twrong\notin\Cand$. RSSM-lite is stronger but remains below OMC-trained
scorers. These are feature-level scorers, not full agents.

\begin{table}[t]
\centering\footnotesize\setlength{\tabcolsep}{5pt}
\caption{General world-model baselines ($3$ seeds, IID; OOD in
Appendix~\ref{app:perseed}). Feature-level scorers.}
\label{tab:baselines}
\begin{tabular}{ll rrrr}
\toprule
Model & Variant & top1 & law\_pair & same$<$wr & mat\_mrr\\
\midrule
  C-SWM-style & OMC-enhanced & 0.358$\pm$0.042 & 0.512$\pm$0.006 & 0.491$\pm$0.012 & 0.212$\pm$0.000 \\
  C-SWM-style & random-neg & 0.214$\pm$0.012 & 0.483$\pm$0.003 & 0.493$\pm$0.006 & 0.211$\pm$0.000 \\
  C-SWM-style & state-corrupt & 0.237$\pm$0.011 & 0.482$\pm$0.007 & 0.492$\pm$0.004 & 0.211$\pm$0.000 \\
  C-SWM-style & temporal-neg & 0.115$\pm$0.017 & 0.513$\pm$0.024 & 0.517$\pm$0.026 & 0.228$\pm$0.015 \\
  RSSM-lite & (no-law obj.) & 0.590$\pm$0.005 & \textbf{0.596}$\pm$0.005 & \textbf{0.609}$\pm$0.007 & \textbf{0.270}$\pm$0.011 \\
\bottomrule
\end{tabular}
\end{table}

\noindent\emph{Clarification (Table~\ref{tab:baselines}).} ``OMC-enhanced'' adds the
wrong-operator negative to the C-SWM-style set. That operator sensitivity still stays
at chance indicates an objective--architecture interaction: the C-SWM-style encoder
does not expose the operator information the negative would exploit.

\section{CoPhy External Validation}
\label{sec:cophy}

The effect transfers to an external counterfactual-physics dataset
(Table~\ref{tab:cophy}, Figure~\ref{fig:summary}). On CoPhy \texttt{ballsCF},
OMC raises wrong-confounder rejection (\textsc{same\_lt\_wrong}
$\approx\!0.74$ vs.\ $0.50$) and operator$\times$trajectory retrieval
(\textsc{matrix\_mrr} $\approx\!0.34$ vs.\ $0.21$) for both backbones, so
operator-blindness and its OMC remedy are not artifacts of the synthetic
generator. CoPhy is used as feature-level external validation.

\begin{table}[t]
\centering\footnotesize\setlength{\tabcolsep}{4.5pt}
\caption{CoPhy \texttt{ballsCF} external validation (\texttt{test}, $5$ seeds).}
\label{tab:cophy}
\begin{tabular}{ll rrrr r}
\toprule
Backbone & Obj. & generic top1 & same$<$wr & full str. & mat\_mrr & diag-gap\\
\midrule
  LOWM & OMC & 0.733$\pm$0.006 & \textbf{0.735}$\pm$0.006 & \textbf{0.714}$\pm$0.014 & \textbf{0.344}$\pm$0.004 & \textbf{1.368}$\pm$0.031 \\
  LOWM & no-law & 0.496$\pm$0.002 & 0.496$\pm$0.002 & 0.449$\pm$0.041 & 0.211$\pm$0.001 & 0.332$\pm$0.215 \\
  direct & OMC & 0.736$\pm$0.014 & \textbf{0.736}$\pm$0.014 & \textbf{0.736}$\pm$0.014 & \textbf{0.309}$\pm$0.007 & 1.044$\pm$0.014 \\
  direct & no-law & 0.504$\pm$0.007 & 0.504$\pm$0.007 & 0.504$\pm$0.007 & 0.210$\pm$0.002 & 0.012$\pm$0.033 \\
\bottomrule
\end{tabular}
\end{table}

\begin{figure}[t]
\centering
\includegraphics[width=\linewidth]{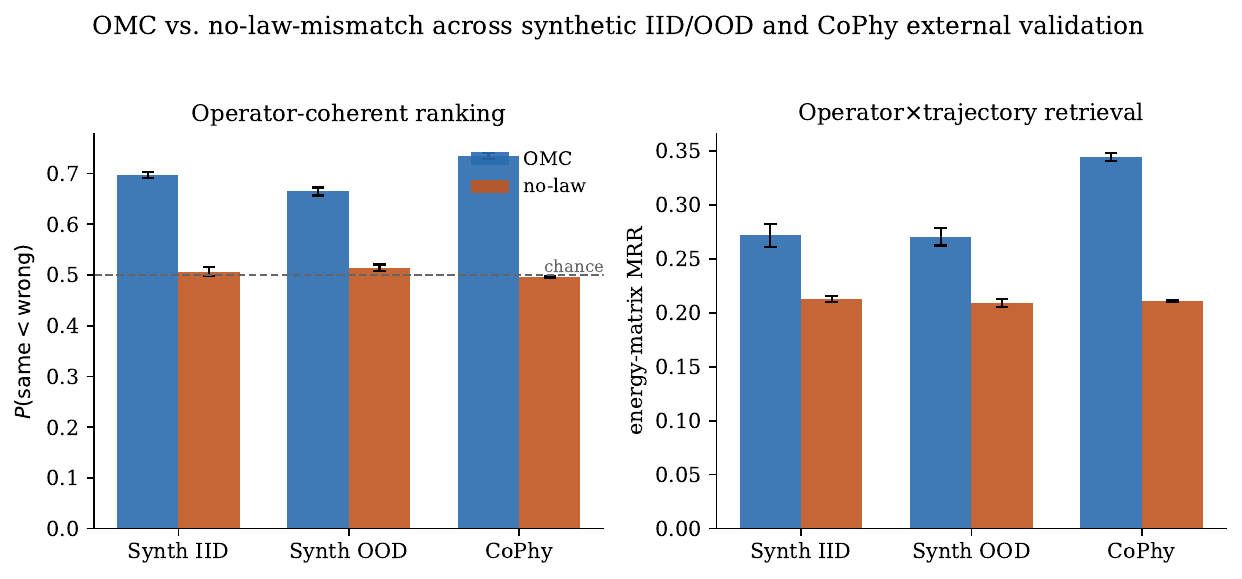}
\caption{\textbf{Summary across settings.} OMC vs.\ no-law on
operator-coherent ranking (left) and energy-matrix retrieval (right) for
synthetic IID/OOD and CoPhy. The baseline stays near chance everywhere; OMC
is consistently above it, including on external CoPhy.}
\label{fig:summary}
\end{figure}

\section{Limitations and Future Directions}
\label{sec:limits}

By design OMC is a \emph{minimal and attachable} coherence signal rather than a
standalone world model: it demonstrates improved \emph{discrimination} of
world-relative coherence---not generation, planning, or control---and its
effect is relational rather than a more decodable operator code. This minimality
is a feature: the entire gain follows from one drop-in change to what counts as a
negative, which makes the signal portable. The pairwise
head is not universally beneficial, \textsc{law\_gap} is comparable only
within a model family (energy scales differ across backbones), and the CoPhy
results are feature-level rather than raw-video. Constructing the
wrong-operator negative requires access to alternative-operator trajectories
at training time, a property of the signal, not of inference.

The coherence signal suggests a generator--critic decomposition: a proposal
model free to entertain plausible-elsewhere futures, paired with OMC as a
critic that reranks them for coherence here. Whether such generativity helps
downstream world models, and whether OMC can serve as an auxiliary verifier
inside latent-dynamics or JEPA-style models, is left to future work.

Three analyses would sharpen the present claims and are immediate next
steps: (i) a finer \emph{operator-distance} study---beyond the threshold
sweep of Section~\ref{sec:sensitivity}---quantifying how same-vs-wrong
separation degrades as the wrong operator approaches the true one; (ii) a
\emph{calibration/verifier} analysis (reliability or AUROC) testing OMC as a
decision rule rather than only a ranking metric; and (iii) a
\emph{data/operator-budget} curve establishing how many alternative-operator
trajectories the signal needs. We regard these as confirmatory rather than
load-bearing: the necessity of the wrong-operator negative is already established
theoretically (Prop.~\ref{prop:order}) and by the ablations of
Section~\ref{sec:synthetic}.

\section{Conclusion}
\label{sec:conclusion}

Plausibility is not coherence. A trajectory scorer can rank generic futures
almost perfectly while remaining blind to which world is in force---a failure
we formalized as operator-blindness and proved invisible to generic ranking.
Operator-Mismatch Contrastive learning makes valid-but-wrong-operator
trajectories the hard negatives, and we showed both theoretically and
empirically---across synthetic IID/OOD, two backbones, and external
CoPhy---that this single change to \emph{what counts as a negative} exposes
and repairs the failure.


\appendix
\section{Notation and Formal Definitions}
\label{app:notation}
Context $C$; operator prior $p(\lambda)$ and conditional $p(\cdot\mid\lambda)$;
posterior $q_\phi(\lambda\mid C)=\mathcal{N}(\mu_\phi(C),\mathrm{diag}\,\sigma^2_\phi(C))$;
operator code $\lamC=\mu_\phi(C)$; energy $\E(\lamC,\tau)$ (lower $=$ more
coherent); candidate classes $\tsame,\twrong,\tnoise$ as in
Section~\ref{sec:formulation}; target ordering Eq.~\eqref{eq:ordering};
operator-blindness Definition~\ref{def:blind}.

\section{Metric Definitions}
\label{app:metrics}
For a sample with candidate energies $\{E_m\}_{m=1}^{M}$ and positive index $y$:
\begin{itemize}
  \item \textbf{top1\_acc} $=\Pr[\arg\min_m E_m=y]$ over the $M{=}5$ set (diagnostic).
  \item \textbf{law\_pair} $=\Pr[\E(\twrong)>\E(\tsame)]$.
  \item \textbf{law\_gap} $=\mathbb{E}[\E(\twrong)-\E(\tsame)]$ (within model family only).
  \item \textbf{law\_only\_top1} $=$ top1\_acc on $\{\tsame,\twrong\}$.
  \item \textbf{same\_lt\_wrong} $=\Pr[\E(\tsame)<\E(\twrong)]$ on $\{\tsame,\twrong,\tnoise\}$.
  \item \textbf{full\_stratification} $=\Pr[\E(\tsame)<\E(\twrong)<\E(\tnoise)]$.
  \item \textbf{Energy matrix} $M_{ij}=\E(\tau_i,\lambda_j)$ over $B{=}16$ distinct-operator
  samples; with $r^{\mathrm{row}}_i,r^{\mathrm{col}}_j$ the diagonal ranks,
  \textbf{matrix\_mrr}$=\tfrac12(\overline{1/r^{\mathrm{row}}}+\overline{1/r^{\mathrm{col}}})$,
  \textbf{matrix\_top1}$=\tfrac12(\overline{[r^{\mathrm{row}}{=}1]}+\overline{[r^{\mathrm{col}}{=}1]})$,
  \textbf{diagonal\_vs\_offdiag\_gap}$=\overline{M_{i\neq j}}-\overline{M_{ii}}$.
\end{itemize}
Definitions match the evaluation code; energy-matrix metrics use
\texttt{matrix\_size}$=16$, \texttt{max\_batches}$=8$.

\section{Objective Details and Proofs}
\label{app:objective}
$\mathcal{L}_{\text{OCCL}}$ is a symmetric InfoNCE over $M_{ij}=\E(\tau_i,\lambda_j)$
(matching trajectories to operators in both directions, temperature $1$); in
all reported runs $\alpha_{\text{occl}}=0$, so it never enters the gradient.
Proofs of Propositions~\ref{prop:insens} and~\ref{prop:order} are given with
their statements in Sections~\ref{sec:formulation} and~\ref{sec:method};
both rely only on (i) the candidate set membership and (ii) the standard
noise-contrastive optimality result \cite{gutmann2010,oord2018}.

\paragraph{Robustness of Proposition~\ref{prop:order} in practice.}
The argument is asymptotic (population limit, expressive scorer, temperature~$1$).
Empirically the induced ordering is stable away from that idealization: it is
essentially unchanged across InfoNCE temperature $T\in[0.5,4]$ and operator-distance
threshold $\delta\in[0.05,0.4]$ (Section~\ref{sec:sensitivity},
Figure~\ref{fig:sensitivity}), and holds even though the candidate set mixes in noise
negatives drawn from outside the data distribution---these are separated trivially and
do not interfere with the same-vs-wrong ordering (Figure~\ref{fig:strat}).

\section{Per-Seed Result Tables}
\label{app:perseed}
All experiments, scenarios, and seeds are reported below
(\texttt{matrix\_size}$=16$, \texttt{max\_batches}$=8$). Splits: IID
(\texttt{test\_iid}), OOD (\texttt{test\_ood\_param}), and \texttt{test}/\texttt{val}
for CoPhy.
\IfFileExists{appendix_per_seed_tables.tex}{%

\subsection{Main synthetic (LOWM and direct-context, OMC vs.\ no-law-mismatch)}
{\scriptsize\setlength{\tabcolsep}{3pt}
\begin{longtable}{llrrrrrrrrr}
\toprule
seed & split & top1 & law\_pair & law\_gap & law\_only & same$<$wr & full str. & mat\_mrr & mat\_t1 & diag-gap \\
\midrule\endfirsthead
\toprule seed & split & top1 & law\_pair & law\_gap & law\_only & same$<$wr & full str. & mat\_mrr & mat\_t1 & diag-gap \\ \midrule\endhead
\bottomrule\endfoot
\multicolumn{11}{l}{\itshape direct\_context\_energy\_direct\_context\_energy} \\
0 & IID & 0.586 & 0.699 & 3.87 & 0.413 & 0.707 & 0.707 & 0.260 & 0.098 & 2.51 \\
1 & IID & 0.402 & 0.490 & -0.01 & 0.199 & 0.499 & 0.499 & 0.211 & 0.062 & 0.00 \\
2 & IID & 0.590 & 0.705 & 4.10 & 0.409 & 0.703 & 0.703 & 0.255 & 0.090 & 2.66 \\
3 & IID & 0.611 & 0.717 & 2.57 & 0.401 & 0.716 & 0.716 & 0.257 & 0.090 & 1.76 \\
4 & IID & 0.583 & 0.704 & 3.94 & 0.414 & 0.716 & 0.716 & 0.257 & 0.098 & 2.56 \\
0 & OOD & 0.567 & 0.690 & 3.67 & 0.397 & 0.713 & 0.713 & 0.262 & 0.102 & 2.47 \\
1 & OOD & 0.391 & 0.489 & -0.04 & 0.208 & 0.490 & 0.490 & 0.211 & 0.062 & 0.00 \\
2 & OOD & 0.563 & 0.684 & 3.98 & 0.391 & 0.706 & 0.706 & 0.255 & 0.094 & 2.64 \\
3 & OOD & 0.583 & 0.687 & 2.58 & 0.385 & 0.712 & 0.712 & 0.255 & 0.090 & 1.75 \\
4 & OOD & 0.563 & 0.689 & 3.69 & 0.376 & 0.704 & 0.704 & 0.260 & 0.102 & 2.48 \\
0 & val & 0.608 & 0.744 & 4.58 & 0.461 & -- & -- & -- & -- & -- \\
1 & val & 0.435 & 0.513 & -0.02 & 0.210 & -- & -- & -- & -- & -- \\
2 & val & 0.621 & 0.749 & 4.87 & 0.461 & -- & -- & -- & -- & -- \\
3 & val & 0.644 & 0.750 & 3.09 & 0.450 & -- & -- & -- & -- & -- \\
4 & val & 0.610 & 0.750 & 4.70 & 0.457 & -- & -- & -- & -- & -- \\
\multicolumn{11}{l}{\itshape lowm\_lowm\_omcr\_no\_pairwise} \\
0 & IID & 0.331 & 0.695 & 0.99 & 0.409 & 0.696 & 0.696 & 0.273 & 0.113 & 0.66 \\
1 & IID & 0.526 & 0.687 & 2.07 & 0.409 & 0.692 & 0.692 & 0.268 & 0.102 & 1.43 \\
2 & IID & 0.294 & 0.700 & 0.83 & 0.409 & 0.690 & 0.690 & 0.265 & 0.094 & 0.55 \\
3 & IID & 0.514 & 0.698 & 1.98 & 0.413 & 0.705 & 0.705 & 0.261 & 0.086 & 1.35 \\
4 & IID & 0.211 & 0.691 & 0.65 & 0.416 & 0.702 & 0.702 & 0.292 & 0.129 & 0.43 \\
0 & OOD & 0.320 & 0.647 & 0.87 & 0.340 & 0.657 & 0.657 & 0.262 & 0.094 & 0.59 \\
1 & OOD & 0.497 & 0.647 & 2.00 & 0.367 & 0.659 & 0.659 & 0.268 & 0.098 & 1.46 \\
2 & OOD & 0.294 & 0.660 & 0.73 & 0.362 & 0.677 & 0.677 & 0.286 & 0.117 & 0.48 \\
3 & OOD & 0.474 & 0.650 & 1.90 & 0.366 & 0.659 & 0.659 & 0.265 & 0.094 & 1.38 \\
4 & OOD & 0.232 & 0.647 & 0.57 & 0.358 & 0.671 & 0.671 & 0.273 & 0.113 & 0.38 \\
0 & val & 0.369 & 0.724 & 1.06 & 0.399 & -- & -- & -- & -- & -- \\
1 & val & 0.563 & 0.727 & 2.41 & 0.430 & -- & -- & -- & -- & -- \\
2 & val & 0.334 & 0.734 & 0.90 & 0.436 & -- & -- & -- & -- & -- \\
3 & val & 0.537 & 0.724 & 2.27 & 0.429 & -- & -- & -- & -- & -- \\
4 & val & 0.270 & 0.730 & 0.70 & 0.431 & -- & -- & -- & -- & -- \\
\multicolumn{11}{l}{\itshape lowm\_no\_law\_mismatch} \\
0 & IID & 0.996 & 0.000 & 0.00 & 0.185 & 0.511 & 0.511 & 0.209 & 0.055 & 1.62 \\
1 & IID & 0.997 & 0.000 & 0.00 & 0.207 & 0.516 & 0.516 & 0.214 & 0.062 & 3.41 \\
2 & IID & 0.994 & 0.000 & 0.00 & 0.192 & 0.490 & 0.490 & 0.211 & 0.059 & 1.53 \\
3 & IID & 0.995 & 0.000 & 0.00 & 0.184 & 0.511 & 0.511 & 0.215 & 0.062 & 0.57 \\
4 & IID & 0.996 & 0.000 & 0.00 & 0.212 & 0.506 & 0.506 & 0.215 & 0.066 & 1.45 \\
0 & OOD & 0.992 & 0.000 & 0.00 & 0.207 & 0.505 & 0.505 & 0.205 & 0.055 & 0.24 \\
1 & OOD & 0.992 & 0.000 & 0.00 & 0.204 & 0.524 & 0.524 & 0.207 & 0.062 & 0.29 \\
2 & OOD & 0.992 & 0.000 & 0.00 & 0.211 & 0.517 & 0.517 & 0.212 & 0.062 & 0.36 \\
3 & OOD & 0.989 & 0.000 & 0.00 & 0.202 & 0.509 & 0.509 & 0.207 & 0.059 & -0.30 \\
4 & OOD & 0.992 & 0.000 & 0.00 & 0.200 & 0.515 & 0.515 & 0.215 & 0.070 & 0.21 \\
0 & val & 0.997 & 0.000 & 0.00 & 0.207 & -- & -- & -- & -- & -- \\
1 & val & 0.995 & 0.000 & 0.00 & 0.241 & -- & -- & -- & -- & -- \\
2 & val & 0.997 & 0.000 & 0.00 & 0.208 & -- & -- & -- & -- & -- \\
3 & val & 0.997 & 0.000 & 0.00 & 0.216 & -- & -- & -- & -- & -- \\
4 & val & 0.995 & 0.000 & 0.00 & 0.203 & -- & -- & -- & -- & -- \\
\end{longtable}}

\subsection{Negative-type ablation: law-mismatch-only}
{\scriptsize\setlength{\tabcolsep}{3pt}
\begin{longtable}{llrrrrrrrrr}
\toprule
seed & split & top1 & law\_pair & law\_gap & law\_only & same$<$wr & full str. & mat\_mrr & mat\_t1 & diag-gap \\
\midrule\endfirsthead
\toprule seed & split & top1 & law\_pair & law\_gap & law\_only & same$<$wr & full str. & mat\_mrr & mat\_t1 & diag-gap \\ \midrule\endhead
\bottomrule\endfoot
\multicolumn{11}{l}{\itshape lowm\_negs\_law\_mismatch\_only\_alpha0\_lambda16} \\
0 & IID & 0.406 & 0.700 & 2.26 & -- & 0.716 & 0.688 & 0.299 & 0.133 & 1.60 \\
1 & IID & 0.419 & 0.695 & 2.18 & -- & 0.716 & 0.675 & 0.277 & 0.105 & 1.52 \\
2 & IID & 0.417 & 0.699 & 2.11 & -- & 0.707 & 0.667 & 0.266 & 0.094 & 1.47 \\
3 & IID & 0.426 & 0.707 & 2.26 & -- & 0.708 & 0.661 & 0.294 & 0.125 & 1.57 \\
4 & IID & 0.412 & 0.697 & 2.04 & -- & 0.713 & 0.688 & 0.282 & 0.113 & 1.43 \\
0 & OOD & 0.399 & 0.699 & 2.20 & -- & 0.699 & 0.675 & 0.309 & 0.129 & 1.48 \\
1 & OOD & 0.401 & 0.690 & 2.11 & -- & 0.698 & 0.666 & 0.269 & 0.090 & 1.41 \\
2 & OOD & 0.405 & 0.692 & 2.01 & -- & 0.690 & 0.654 & 0.280 & 0.102 & 1.35 \\
3 & OOD & 0.403 & 0.696 & 2.25 & -- & 0.685 & 0.639 & 0.284 & 0.102 & 1.55 \\
4 & OOD & 0.406 & 0.697 & 1.95 & -- & 0.688 & 0.667 & 0.289 & 0.109 & 1.31 \\
0 & val & 0.445 & 0.741 & 2.84 & 0.445 & -- & -- & -- & -- & -- \\
1 & val & 0.441 & 0.737 & 2.73 & 0.441 & -- & -- & -- & -- & -- \\
2 & val & 0.452 & 0.738 & 2.63 & 0.452 & -- & -- & -- & -- & -- \\
3 & val & 0.430 & 0.733 & 2.82 & 0.430 & -- & -- & -- & -- & -- \\
4 & val & 0.439 & 0.737 & 2.53 & 0.439 & -- & -- & -- & -- & -- \\
\end{longtable}}

\subsection{Component ablation: \texttt{main\_all} (pairwise energy on)}
{\scriptsize\setlength{\tabcolsep}{3pt}
\begin{longtable}{llrrrrrrrrr}
\toprule
seed & split & top1 & law\_pair & law\_gap & law\_only & same$<$wr & full str. & mat\_mrr & mat\_t1 & diag-gap \\
\midrule\endfirsthead
\toprule seed & split & top1 & law\_pair & law\_gap & law\_only & same$<$wr & full str. & mat\_mrr & mat\_t1 & diag-gap \\ \midrule\endhead
\bottomrule\endfoot
\multicolumn{11}{l}{\itshape lowm\_component\_main\_all} \\
0 & IID & 0.508 & 0.691 & 1.61 & -- & 0.701 & 0.701 & 0.250 & 0.090 & 1.13 \\
1 & IID & 0.520 & 0.696 & 1.71 & -- & 0.696 & 0.696 & 0.264 & 0.102 & 1.20 \\
2 & IID & 0.508 & 0.688 & 1.63 & -- & 0.705 & 0.705 & 0.252 & 0.086 & 1.15 \\
3 & IID & 0.492 & 0.691 & 1.72 & -- & 0.703 & 0.703 & 0.246 & 0.082 & 1.20 \\
4 & IID & 0.518 & 0.692 & 1.64 & -- & 0.708 & 0.708 & 0.261 & 0.102 & 1.17 \\
0 & OOD & 0.477 & 0.651 & 1.54 & -- & 0.669 & 0.669 & 0.258 & 0.094 & 1.18 \\
1 & OOD & 0.495 & 0.659 & 1.64 & -- & 0.671 & 0.671 & 0.261 & 0.102 & 1.22 \\
2 & OOD & 0.489 & 0.661 & 1.57 & -- & 0.675 & 0.675 & 0.253 & 0.090 & 1.18 \\
3 & OOD & 0.463 & 0.653 & 1.63 & -- & 0.667 & 0.667 & 0.253 & 0.094 & 1.21 \\
4 & OOD & 0.487 & 0.656 & 1.60 & -- & 0.670 & 0.670 & 0.260 & 0.102 & 1.18 \\
0 & val & 0.524 & 0.726 & 2.09 & 0.411 & -- & -- & -- & -- & -- \\
1 & val & 0.527 & 0.719 & 2.24 & 0.406 & -- & -- & -- & -- & -- \\
2 & val & 0.532 & 0.722 & 2.14 & 0.413 & -- & -- & -- & -- & -- \\
3 & val & 0.500 & 0.721 & 2.24 & 0.395 & -- & -- & -- & -- & -- \\
4 & val & 0.532 & 0.730 & 2.14 & 0.417 & -- & -- & -- & -- & -- \\
\end{longtable}}

\subsection{Component ablation: \texttt{no\_pairwise} (pairwise energy off)}
{\scriptsize\setlength{\tabcolsep}{3pt}
\begin{longtable}{llrrrrrrrrr}
\toprule
seed & split & top1 & law\_pair & law\_gap & law\_only & same$<$wr & full str. & mat\_mrr & mat\_t1 & diag-gap \\
\midrule\endfirsthead
\toprule seed & split & top1 & law\_pair & law\_gap & law\_only & same$<$wr & full str. & mat\_mrr & mat\_t1 & diag-gap \\ \midrule\endhead
\bottomrule\endfoot
\multicolumn{11}{l}{\itshape lowm\_component\_no\_pairwise} \\
0 & IID & 0.331 & 0.695 & 0.99 & -- & 0.696 & 0.696 & 0.273 & 0.113 & 0.66 \\
1 & IID & 0.526 & 0.687 & 2.07 & -- & 0.692 & 0.692 & 0.268 & 0.102 & 1.43 \\
2 & IID & 0.294 & 0.700 & 0.83 & -- & 0.690 & 0.690 & 0.265 & 0.094 & 0.55 \\
3 & IID & 0.514 & 0.698 & 1.98 & -- & 0.705 & 0.705 & 0.261 & 0.086 & 1.35 \\
4 & IID & 0.210 & 0.691 & 0.65 & -- & 0.702 & 0.702 & 0.292 & 0.129 & 0.43 \\
0 & OOD & 0.320 & 0.647 & 0.87 & -- & 0.657 & 0.657 & 0.262 & 0.094 & 0.59 \\
1 & OOD & 0.497 & 0.647 & 2.00 & -- & 0.659 & 0.659 & 0.268 & 0.098 & 1.46 \\
2 & OOD & 0.294 & 0.660 & 0.73 & -- & 0.677 & 0.677 & 0.286 & 0.117 & 0.48 \\
3 & OOD & 0.474 & 0.650 & 1.90 & -- & 0.659 & 0.659 & 0.265 & 0.094 & 1.38 \\
4 & OOD & 0.232 & 0.647 & 0.57 & -- & 0.671 & 0.671 & 0.273 & 0.113 & 0.38 \\
0 & val & 0.369 & 0.724 & 1.06 & 0.399 & -- & -- & -- & -- & -- \\
1 & val & 0.563 & 0.727 & 2.41 & 0.430 & -- & -- & -- & -- & -- \\
2 & val & 0.334 & 0.734 & 0.90 & 0.436 & -- & -- & -- & -- & -- \\
3 & val & 0.537 & 0.724 & 2.27 & 0.429 & -- & -- & -- & -- & -- \\
4 & val & 0.270 & 0.730 & 0.70 & 0.431 & -- & -- & -- & -- & -- \\
\end{longtable}}

\subsection{C-SWM-style baselines (feature-level)}
{\scriptsize\setlength{\tabcolsep}{3pt}
\begin{longtable}{llrrrrrrrrr}
\toprule
seed & split & top1 & law\_pair & law\_gap & law\_only & same$<$wr & full str. & mat\_mrr & mat\_t1 & diag-gap \\
\midrule\endfirsthead
\toprule seed & split & top1 & law\_pair & law\_gap & law\_only & same$<$wr & full str. & mat\_mrr & mat\_t1 & diag-gap \\ \midrule\endhead
\bottomrule\endfoot
\multicolumn{11}{l}{\itshape cswm\_style\_energy\_cswm\_plus\_omc} \\
0 & IID & 0.348 & 0.505 & 0.01 & 0.194 & 0.485 & 0.485 & 0.212 & 0.062 & 0.00 \\
1 & IID & 0.404 & 0.517 & 0.02 & 0.196 & 0.505 & 0.505 & 0.211 & 0.062 & -0.00 \\
2 & IID & 0.322 & 0.514 & 0.01 & 0.219 & 0.484 & 0.484 & 0.211 & 0.062 & 0.00 \\
0 & OOD & 0.337 & 0.502 & -0.03 & 0.199 & 0.485 & 0.485 & 0.212 & 0.062 & 0.00 \\
1 & OOD & 0.364 & 0.516 & 0.01 & 0.205 & 0.500 & 0.500 & 0.211 & 0.062 & -0.00 \\
2 & OOD & 0.299 & 0.496 & 0.01 & 0.225 & 0.492 & 0.492 & 0.212 & 0.062 & 0.00 \\
\multicolumn{11}{l}{\itshape cswm\_style\_energy\_cswm\_random\_neg} \\
0 & IID & 0.204 & 0.481 & -2.59 & 0.188 & 0.488 & 0.078 & 0.211 & 0.062 & 0.16 \\
1 & IID & 0.227 & 0.486 & -3.80 & 0.184 & 0.492 & 0.058 & 0.211 & 0.062 & 0.08 \\
2 & IID & 0.212 & 0.483 & -2.26 & 0.193 & 0.499 & 0.109 & 0.211 & 0.062 & -0.03 \\
0 & OOD & 0.213 & 0.467 & -6.66 & 0.197 & 0.512 & 0.088 & 0.211 & 0.062 & 0.29 \\
1 & OOD & 0.210 & 0.458 & -9.25 & 0.199 & 0.510 & 0.060 & 0.211 & 0.062 & 0.63 \\
2 & OOD & 0.225 & 0.475 & -3.96 & 0.204 & 0.501 & 0.087 & 0.211 & 0.062 & -0.08 \\
\multicolumn{11}{l}{\itshape cswm\_style\_energy\_cswm\_state\_corrupted\_neg} \\
0 & IID & 0.244 & 0.490 & -36.80 & 0.195 & 0.489 & 0.375 & 0.211 & 0.062 & 2.51 \\
1 & IID & 0.242 & 0.480 & -99.12 & 0.182 & 0.496 & 0.157 & 0.211 & 0.062 & 4.47 \\
2 & IID & 0.224 & 0.477 & -74.54 & 0.189 & 0.492 & 0.216 & 0.211 & 0.062 & 4.33 \\
0 & OOD & 0.222 & 0.476 & -100.08 & 0.195 & 0.511 & 0.391 & 0.212 & 0.062 & 3.46 \\
1 & OOD & 0.221 & 0.462 & -292.88 & 0.202 & 0.504 & 0.160 & 0.211 & 0.062 & 13.84 \\
2 & OOD & 0.204 & 0.463 & -205.58 & 0.199 & 0.513 & 0.225 & 0.211 & 0.062 & 7.40 \\
\multicolumn{11}{l}{\itshape cswm\_style\_energy\_cswm\_temporal\_neg} \\
0 & IID & 0.103 & 0.514 & 1297.24 & 0.238 & 0.499 & 0.002 & 0.234 & 0.086 & 610.12 \\
1 & IID & 0.108 & 0.489 & 3862.39 & 0.206 & 0.504 & 0.000 & 0.211 & 0.066 & 1398.84 \\
2 & IID & 0.134 & 0.537 & 6644.52 & 0.245 & 0.547 & 0.001 & 0.240 & 0.082 & 6941.41 \\
0 & OOD & 0.129 & 0.495 & -119.40 & 0.215 & 0.514 & 0.002 & 0.206 & 0.055 & 169.49 \\
1 & OOD & 0.114 & 0.493 & 3586.63 & 0.200 & 0.501 & 0.000 & 0.212 & 0.066 & 2645.30 \\
2 & OOD & 0.171 & 0.547 & 8494.51 & 0.238 & 0.558 & 0.000 & 0.238 & 0.086 & 2536.95 \\
\end{longtable}}

\subsection{RSSM-lite baseline (feature-level)}
{\scriptsize\setlength{\tabcolsep}{3pt}
\begin{longtable}{llrrrrrrrrr}
\toprule
seed & split & top1 & law\_pair & law\_gap & law\_only & same$<$wr & full str. & mat\_mrr & mat\_t1 & diag-gap \\
\midrule\endfirsthead
\toprule seed & split & top1 & law\_pair & law\_gap & law\_only & same$<$wr & full str. & mat\_mrr & mat\_t1 & diag-gap \\ \midrule\endhead
\bottomrule\endfoot
\multicolumn{11}{l}{\itshape rssm\_lite\_scorer\_rssm\_lite\_mse} \\
0 & IID & 0.596 & 0.602 & 0.00 & 0.282 & 0.608 & 0.608 & 0.269 & 0.113 & 0.00 \\
1 & IID & 0.588 & 0.594 & 0.00 & 0.284 & 0.616 & 0.616 & 0.281 & 0.133 & 0.00 \\
2 & IID & 0.587 & 0.592 & 0.00 & 0.281 & 0.602 & 0.602 & 0.259 & 0.102 & 0.00 \\
0 & OOD & 0.548 & 0.557 & 0.00 & 0.296 & 0.571 & 0.571 & 0.260 & 0.086 & 0.00 \\
1 & OOD & 0.548 & 0.557 & 0.00 & 0.293 & 0.567 & 0.567 & 0.262 & 0.090 & 0.00 \\
2 & OOD & 0.549 & 0.557 & 0.00 & 0.284 & 0.574 & 0.574 & 0.248 & 0.070 & 0.00 \\
\end{longtable}}

\subsection{CoPhy \texttt{ballsCF} external validation (feature-level)}
{\scriptsize\setlength{\tabcolsep}{3pt}
\begin{longtable}{llrrrrrrrrr}
\toprule
seed & split & top1 & law\_pair & law\_gap & law\_only & same$<$wr & full str. & mat\_mrr & mat\_t1 & diag-gap \\
\midrule\endfirsthead
\toprule seed & split & top1 & law\_pair & law\_gap & law\_only & same$<$wr & full str. & mat\_mrr & mat\_t1 & diag-gap \\ \midrule\endhead
\bottomrule\endfoot
\multicolumn{11}{l}{\itshape direct\_context\_energy\_direct\_context\_cophy\_OMC} \\
0 & test & 0.730 & 0.730 & 1.07 & 0.406 & -- & -- & 0.314 & 0.137 & 1.23 \\
1 & test & 0.740 & 0.740 & 1.08 & 0.424 & -- & -- & 0.310 & 0.125 & 1.23 \\
2 & test & 0.728 & 0.730 & 1.05 & 0.402 & -- & -- & 0.310 & 0.125 & 1.23 \\
3 & test & 0.756 & 0.758 & 1.08 & 0.460 & -- & -- & 0.334 & 0.145 & 1.28 \\
4 & test & 0.732 & 0.732 & 1.07 & 0.386 & -- & -- & 0.329 & 0.152 & 1.26 \\
0 & val & 0.700 & 0.706 & 1.00 & 0.424 & -- & -- & 0.315 & 0.145 & 1.05 \\
1 & val & 0.704 & 0.712 & 1.02 & 0.410 & -- & -- & 0.308 & 0.125 & 1.07 \\
2 & val & 0.704 & 0.712 & 1.00 & 0.412 & -- & -- & 0.294 & 0.109 & 1.08 \\
3 & val & 0.724 & 0.732 & 1.06 & 0.444 & -- & -- & 0.323 & 0.141 & 1.12 \\
4 & val & 0.720 & 0.728 & 1.03 & 0.428 & -- & -- & 0.304 & 0.129 & 1.08 \\
\multicolumn{11}{l}{\itshape direct\_context\_energy\_direct\_context\_cophy\_no\_law\_mismatch} \\
0 & test & 0.996 & 0.000 & 0.00 & 0.188 & -- & -- & 0.212 & 0.062 & -0.11 \\
1 & test & 0.996 & 0.000 & 0.00 & 0.198 & -- & -- & 0.212 & 0.062 & -0.14 \\
2 & test & 0.996 & 0.000 & 0.00 & 0.186 & -- & -- & 0.210 & 0.059 & -0.11 \\
3 & test & 0.996 & 0.000 & 0.00 & 0.188 & -- & -- & 0.209 & 0.059 & -0.17 \\
4 & test & 0.996 & 0.000 & 0.00 & 0.192 & -- & -- & 0.208 & 0.059 & -0.13 \\
0 & val & 0.978 & 0.000 & 0.00 & 0.172 & -- & -- & 0.213 & 0.062 & 0.14 \\
1 & val & 0.978 & 0.000 & 0.00 & 0.180 & -- & -- & 0.212 & 0.062 & 0.10 \\
2 & val & 0.980 & 0.000 & 0.00 & 0.180 & -- & -- & 0.213 & 0.062 & 0.11 \\
3 & val & 0.980 & 0.000 & 0.00 & 0.188 & -- & -- & 0.209 & 0.059 & 0.14 \\
4 & val & 0.980 & 0.000 & 0.00 & 0.184 & -- & -- & 0.213 & 0.062 & 0.10 \\
\multicolumn{11}{l}{\itshape lowm\_lowm\_cophy\_OMC} \\
0 & test & 0.738 & 0.752 & 1.44 & 0.426 & -- & -- & 0.365 & 0.160 & 1.74 \\
1 & test & 0.724 & 0.740 & 1.44 & 0.420 & -- & -- & 0.366 & 0.164 & 1.76 \\
2 & test & 0.720 & 0.738 & 1.54 & 0.418 & -- & -- & 0.364 & 0.156 & 1.84 \\
3 & test & 0.730 & 0.748 & 1.45 & 0.430 & -- & -- & 0.372 & 0.168 & 1.75 \\
4 & test & 0.748 & 0.756 & 1.46 & 0.438 & -- & -- & 0.366 & 0.160 & 1.74 \\
0 & val & 0.720 & 0.742 & 1.41 & 0.430 & -- & -- & 0.372 & 0.184 & 1.53 \\
1 & val & 0.716 & 0.744 & 1.45 & 0.426 & -- & -- & 0.363 & 0.168 & 1.55 \\
2 & val & 0.736 & 0.762 & 1.53 & 0.436 & -- & -- & 0.379 & 0.191 & 1.63 \\
3 & val & 0.720 & 0.742 & 1.43 & 0.430 & -- & -- & 0.374 & 0.184 & 1.55 \\
4 & val & 0.722 & 0.740 & 1.45 & 0.446 & -- & -- & 0.363 & 0.172 & 1.55 \\
\multicolumn{11}{l}{\itshape lowm\_lowm\_cophy\_no\_law\_mismatch} \\
0 & test & 0.978 & 0.000 & 0.00 & 0.202 & -- & -- & 0.212 & 0.062 & 1.09 \\
1 & test & 0.982 & 0.000 & 0.00 & 0.198 & -- & -- & 0.211 & 0.062 & 0.66 \\
2 & test & 0.978 & 0.000 & 0.00 & 0.190 & -- & -- & 0.210 & 0.059 & 0.23 \\
3 & test & 0.982 & 0.000 & 0.00 & 0.194 & -- & -- & 0.213 & 0.066 & 0.56 \\
4 & test & 0.988 & 0.000 & 0.00 & 0.190 & -- & -- & 0.212 & 0.062 & 0.41 \\
0 & val & 0.970 & 0.000 & 0.00 & 0.204 & -- & -- & 0.214 & 0.066 & 1.33 \\
1 & val & 0.972 & 0.000 & 0.00 & 0.208 & -- & -- & 0.213 & 0.066 & 0.81 \\
2 & val & 0.968 & 0.000 & 0.00 & 0.216 & -- & -- & 0.214 & 0.066 & 0.37 \\
3 & val & 0.970 & 0.000 & 0.00 & 0.210 & -- & -- & 0.213 & 0.066 & 0.74 \\
4 & val & 0.970 & 0.000 & 0.00 & 0.212 & -- & -- & 0.213 & 0.066 & 0.59 \\
\end{longtable}}
}{\emph{[Per-seed tables file \texttt{appendix\_per\_seed\_tables.tex} not found in this build.]}}

\section{Reproducibility}
\label{app:repro}
Optimizer AdamW \cite{loshchilov2019} (lr $10^{-3}$, weight decay $10^{-4}$);
$\beta_{\text{KL}}{=}10^{-4}$, $\alpha_{\text{stable}}{=}0.1$,
$\alpha_{\text{occl}}{=}0$, OCCL temperature $1$; checkpoint by best
\textsc{law\_pair}. Candidate set $M{=}5$ (one negative per type), operator
distance threshold $\delta{=}0.15$, distinct operators per batch;
$d_\lambda{=}16$, hidden/context dim $128$, $2$ layers, sum energy reduction.
Synthetic: $30$ epochs, batch $64$, $5000$/$1000$ train/val per epoch, seeds
$0$--$4$. CoPhy: $30$ epochs, batch $2048$, $20000$/$3000$, seeds $0$--$4$.
Baselines (C-SWM-style, RSSM-lite): $30$ epochs, batch $1024$,
$20000$/$3000$, seeds $0$--$2$. Energy-matrix eval:
\texttt{matrix\_size}$=16$, \texttt{max\_batches}$=8$. Hardware: Colab GPU
(T4/L4-class).

\section{CoPhy Preprocessing}
\label{app:cophy}
CoPhy \texttt{ballsCF} ($224$px) is reduced to segmentation-derived object
features at $64$px over $20$ frames; the hidden confounder serves as the
operator metadata for negatives and evaluation. Episode split:
$35{,}000$/$7{,}500$/$7{,}500$ (train/val/test), total $50{,}000$.

\end{document}